\documentclass[12pt]{amsart}
\usepackage[a4paper,hmargin=2cm,vmargin=3cm]{geometry} 
\usepackage{amsthm,mathtools}
\usepackage{amssymb,amsaddr}
\usepackage{amsmath}
\usepackage{enumitem}
\usepackage{tikz}
\usepackage{caption}
\usepackage{wrapfig}
\usepackage{mathrsfs,bm,color}
\usepackage{graphicx}
\usepackage{hyperref}
\hypersetup{
	colorlinks,
	linkcolor={red!50!black},
	citecolor={green!50!black},
	urlcolor={blue!80!black}
}

\newcommand{\trans}{\mathsf{T}}

\newcommand{\R}{\mathbb{R}}
\newcommand{\bv}{\bm{v}}
\newcommand{\bx}{\bm{x}}
\newcommand{\bd}{\operatorname{blockdiag}}
\newcommand{\bp}{\operatorname{blockproj}}

\numberwithin{equation}{section}

\theoremstyle{plain}
\newtheorem{theorem}{Theorem}[section]
\newtheorem{lemma}[theorem]{Lemma}

\newtheorem{definition}[theorem]{Definition}
\newtheorem{conjecture}[theorem]{Conjecture}

\theoremstyle{remark}
\newtheorem{remark}[theorem]{Remark}

\newtheorem{assumption}[theorem]{Assumption}

\author{Aleksandr Pak$^{1,2}$, Justin Ko$^1$, Florent Krzakala$^2$}
\address{$^1$. \'Ecole Normale Sup\'erieure de Lyon, France, $^2$.\'Ecole Polytechnique F\'ed\'erale de Lausanne (EPFL),  IdePHICS Lab,   CH-1015 Lausanne, Switzerland}

\title[Optimal Algorithms for the Inhomogeneous Spiked Wigner Model]{Optimal Algorithms for the Inhomogeneous Spiked Wigner Model}
\usepackage{times}
	\email{justin.ko@ens-lyon.fr,florent.krzakala@epfl.ch,aleksandr.pak@epfl.ch}

\begin{document}

\maketitle

\begin{abstract}%
  In this paper, we study a spiked Wigner problem with an inhomogeneous noise profile. Our aim in this problem is to recover the signal passed through an inhomogeneous low-rank matrix channel. While the information-theoretic performances are well-known, we focus on the algorithmic problem. We derive an approximate message-passing algorithm (AMP) for the inhomogeneous problem and show that its rigorous state evolution coincides with the information-theoretic optimal Bayes fixed-point equations. We identify in particular the existence of a statistical-to-computational gap where known algorithms require a signal-to-noise ratio bigger than the information-theoretic threshold to perform better than random. Finally, from the adapted AMP iteration we deduce a simple and efficient spectral method that can be used to recover the transition for matrices with general variance profiles. This spectral method matches the conjectured optimal computational phase transition.
\end{abstract}

\section{Introduction}
Low-rank information extraction from a noisy data matrix is a crucial statistical challenge. The spiked random matrix models have recently gained extensive interest in the fields of statistics, probability, and machine learning, serving as a valuable platform for exploring this issue \cite{donoho1995adapting,peche2014deformed,BBP,lesieur2017constrained}. A prominent example is the spiked Wigner model, where a rank one matrix is observed through a component-wise homogeneous noise.

Heterogeneity being a fundamental part of many real-world problems,  we consider here an inhomogeneous version of the Wigner spike model, discussed in \cite{AJFL_inhomo,behne2022fundamental}, where  the signal is observed through an inhomogeneous, block-constant noise. Consider a partition $\{ 1, \ldots, N \} \!=\! [N]$ into $q$ disjoint groups $C_{1}^{N}  \cup \cdots \cup C_{q}^{N}=[N]$. This partition is encoded by a function $g:\! [N]\! \mapsto \![q]$ which maps each index $i \in [N]$ into its group $g(i) \!\in\! [q]$.   Let $\tilde{\boldsymbol{\Delta}} \!\in\! \R^{q \times q}$ be a 
symmetric matrix encoding a block-constant symmetric matrix $\boldsymbol{\Delta}\! \in\! \R^{N \times N}$
\begin{equation}\label{eq:varprofile}
\boldsymbol{\Delta}_{ij} = \tilde{\boldsymbol{\Delta}}_{g(i)g(j)}.
\end{equation}
We observe the signal $\boldsymbol{x}^{\star} \in \mathbb{R}^{N}$ which is assumed to have independent identically distributed coordinates generated from some prior distribution $\mathbb{P} _{0}$ (i.e. $\mathbb{P} ( \boldsymbol{x}^{\star} = \boldsymbol{x} ) = \prod_{i = 1}^{N}\mathbb{P}_{0}  (x_{i}^{\star} = x_{i})$) through  noisy measurements:
\begin{equation} \label{eqn: low-rank inhomogenous}
		\boldsymbol{Y} = \sqrt{\frac{1}{N}}\boldsymbol{x}^{\star}(\boldsymbol{x}^{\star})^{T} + \boldsymbol{A} \odot \sqrt{\boldsymbol{\Delta}}.
\end{equation}
Here and throughout the article $\odot$ denotes the Hadamard product, $\sqrt{\boldsymbol{\Delta}}$ is the Hadamard square-root of $\boldsymbol{\Delta}$ and $\boldsymbol{A}$ is a real-valued symmetric GOE matrix with off-diagonal elements of unit variance.  The Bayes-optimal performance of this model in the asymptotic limit $N \to \infty$ was studied rigorously in \cite{AJFL_inhomo,behne2022fundamental,MourratXia-tensor, MourratXiaChen-tensor} who characterized the fundamental information-theoretic limit of reconstruction in this model. Here we focus instead on the algorithmic problem of reconstructing the (hidden) spike. {\bf Our contributions are many-fold}:
\begin{itemize}[wide=1pt, topsep=0pt,nosep]
    \item We show how one can construct an Approximate Message Passing (AMP) algorithm for the inhomogeneous Wigner problem, whose asymptotic performance can be tracked by a rigorous state evolution, generalizing the homogeneous version of the algorithm for low-rank factorization \cite{BayatiMontanari,DBLP:journals/corr/DeshpandeAM15,lesieur2017constrained}. 
    \item  AMP is shown to give Bayes optimal performances, as characterized in \cite{AJFL_inhomo}, for a wide choice of parameters. There exists, however, a region of parameters where AMP differs from Bayes performances, yielding a computational-to-statistical gap \cite{bandeira2018notes,celentano2020estimation}. In this region, we conjecture that the problem is hard for a large class of algorithms. 
    \item  Finally, we present a linear version of AMP \cite{maillard2022construction}, that turns out to be equivalent to a spectral method, which is optimal in the sense that it can detect the presence of the spike in the same region as AMP. This is quite remarkable since, as shown in \cite[Section~2.4]{AJFL_inhomo}, the standard spectral method (PCA) applied to a simple renormalization of the matrix fails to do so.
\end{itemize}

\paragraph{\textbf{Related work}}
The class of  approximate message passing algorithms (AMP) has attracted a lot of attention in the high-dimensional statistics and machine learning community, see e.g.  \cite{donoho2009message,BayatiMontanari,rangan2011generalized,DBLP:journals/corr/DeshpandeAM15,lesieur2017constrained,gerbelot2021graph,feng2022unifying}. The ideas behind this algorithm have roots in physics of spin glasses \cite{mezard1987spin,bolthausen2014iterative,zdeborova2016statistical}. AMP algorithms are optimal among first order methods \cite{celentano2020estimation}, thus their reconstruction threshold provides a bound on the algorithmic complexity  in our model. Our approach to the inhomogeneous version of AMP relies on several refinements of AMP methods to handle the full complexity of the problem, notably the 
spatial coupling technique \cite{krzakala2012probabilistic,donoho2013information,8205391,gerbelot2021graph}.

Factorizing low-rank matrices is a ubiquitous problem with many applications in machine learning and statistics, ranging from sparse PCA to community detection and sub-matrix localization. Many variants of the homogeneous problem have been studied in the high-dimensional limit \cite{DBLP:journals/corr/DeshpandeAM15,lesieur2017constrained,barbier2018rank,lesieur2017constrained,10.1214/19-AOS1826,lelargemiolanematrixestimation,barbier2020information}. The inhomogeneous version was discussed in details in \cite{AJFL_inhomo,behne2022fundamental}. Spectral methods are a very popular tool to solve rank-factorization problems \cite{donoho1995adapting,peche2014deformed,BBP}.  Using AMP as an inspiration for deriving new spectral methods was discussed, for instance, in \cite{saade2014spectral,lesieur2017constrained,aubin2019spiked,mondelli2018fundamental,mondelli2022optimal,maillard2022construction,venkataramanan2022estimation}.

\section{Main results}


\paragraph{\textbf{Message passing algorithm}} 
For each $t \geq 0$, let $(f_t^a)_{a \in [q]}$ be a collection of Lipschitz functions from $\R \to \R$, and define $f_{t}: \mathbb{R}^{N} \times \mathbb{N} \mapsto \mathbb{R}^{N}$ by
\[
f_t(\bm{x}) := ( f_t^{g(i)} (x_i) )_{i \in [N]} \in \R^N.
\]
These linear functions are often called denoiser functions and can be chosen amongst several options, such as the Bayes optimal denoisers for practical applications (see Section~\ref{sec:Bayes}), or even linear denoisers (see Section~\ref{sec:linear}). We shall consider the following AMP recursion for an estimator in the inhomogeneous spiked Wigner problem
\begin{equation}
\label{eqn: AMP spike}
\boldsymbol{x}^{t+1} = \left(\frac{1}{\sqrt{N}\bm{\Delta}} \odot \bm{Y}\right) f_{t}\left(\boldsymbol{x}^{t}\right)-\bm{\mathrm{b}}_{t} \odot f_{t-1}\left(\boldsymbol{x}^{t-1}\right)
\end{equation}
with the so-called Onsager term $\bm{\mathrm{b}}_{t} = \frac{1}{\bm{\Delta}} f_{t}^{\prime}(x^{t}) \in \R^N$ 
where $\frac{1}{\bm{\Delta}}$ is the Hadamard inverse of $\bm{\Delta}$ and $f_{t}^{\prime}$ is the vector of coordinate wise derivatives. 

In practical implementations, we initialize the algorithm with some non-null $\boldsymbol{x}^{0}$ and let it run for a certain number of iterations. One efficient way to do this is the spectral initialization \cite{mondelli2021approximate} with the method described in sec. \ref{sec:spec}. In Figure~\ref{fig:dense_general} we provide an example of the performance of the AMP together with the Bayes-optimal estimator predicted by the asymptotic theory. Even at very moderate sizes, the agreement is clear.

\paragraph{\textbf{State evolution}}
Our first main contribution  is the generalisation of the state evolution characterization of the behaviour of AMP \cite[Theorem~1]{8205391} in the inhomogeneous setting. To state a well-defined limit of the AMP, we have the following assumptions. 

\begin{assumption} \label{ass:limit}
To ensure that our inhomogeneous AMP has a well-defined limit, we assume that
\begin{enumerate}
    \item For each $a \in [q]$, we have
    \[
    \lim_{N \to \infty} \frac{|C_a^N|}{N} \to c_a \in (0,1).
    \]
    \item The family of real valued functions such that $(f_t^a)_{a \in [q]}$ and $(f_t^a)^\prime_{a \in [q]}$ are Lipschitz.
    \item For each $a \in [q]$, there exists $\left(\sigma^0_a\right)^{2} \in \R$ such that, in probability,
    \[
    \lim_{N \to \infty} \frac{1}{|C_a^N|} \sum_{i \in C_a^N} f_0^a(x^0_i) f_0^a(x^0_i) = \left(\sigma^0_a\right)^{2} .
    \]
\end{enumerate}
\end{assumption}

Our first result describes the distribution of the iterates in the limit. Our mode of convergence will be with respect to $L$-pseudo-Lipschitz test functions $\phi: \R^M \to \R$ satisfying
    \begin{equation}\label{eq:pseudoLip}
| \phi(x) - \phi(y)| \leq L (1 + \|x\| + \| y\|) \| x- y\| \qquad \text{for all $x,y \in \R^M$}.
    \end{equation}

We define the following state evolution parameters $\mu_{b}^{t}$ and $\sigma_{b}^{t}$ for $b \in [q]$ through the recursion
    \begin{equation}\label{eq:state}
        \begin{aligned} 
            & \mu^{t + 1}_{b} = \sum_{a \in [q]} \frac{c_{a}}{\tilde{\bm{\Delta}}_{ab}} \mathbb{E}_{x_{0}^{\star}, Z}[x_{0}^{\star} f_{t}^{a}\left(\mu^{t}_{a}x_{0}^{\star} + \sigma^{t}_{a}Z\right)] ~\text{with}~ x_{0}^{\star} \sim \mathbb{P}_{0}, Z \sim \mathcal{N} (0, 1) \\
            & (\sigma_{b}^{t + 1})^{2} = \sum_{a=1}^{q} \frac{c_{a}} {\tilde{\bm{\Delta}}_{ab}}\mathbb{E}_{x_{0}^{\star}, Z}\left[ (f_{t}^{a}(\mu^{t}_{a}x_{0}^{\star} +  \sigma_{a}^{t}Z))^{2}\right] ~\text{with}~ x_{0}^{\star} \sim \mathbb{P}_{0}, Z \sim \mathcal{N} (0, 1),
        \end{aligned}
    \end{equation}
where $x_0^\star$ and $Z$ are independent. We use the initialization $\bm{\mu}^0 = \bm{\sigma}^0 = 0$. We prove that the iterates $x_i^{t}$ are asymptotically equal in distribution to $\mu_{g(i)}^{t}x_{0}^{\star} + \sigma_{g(i)}^{t}Z$ where $x^\star_0$ and  $Z$ are independent.  
\begin{wrapfigure}{rt}{0.5\textwidth}
  \begin{center}
    \includegraphics[width=0.47\textwidth]{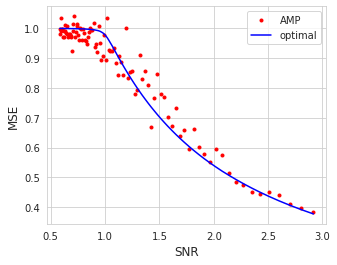} 
  \end{center}
  \vspace{-0.5cm}
  \caption{Performance of the inhomogeneous AMP algorithm against the information-theoretical optimal {\rm MMSE}. The variance profile is proportional to $\tilde{\bm{\Delta}} \!=\! \begin{bmatrix}
		    1 & 3\\
      3 & 2
		\end{bmatrix}$ 
with two equally sized blocks with standard Gaussian prior when $N \!=\! 500$ at various {\rm snr}. }
		\label{fig:dense_general}
    \vspace{-1.cm}
    \end{wrapfigure}

\begin{theorem}[State evolution of AMP iterates in the inhomogeneous setting] \label{theorem: AMP inhomogeneous}
    Suppose that Assumption~\ref{ass:limit} holds, and that $\mathbb{P}_0$ has bounded second moment. Let $\phi: \R^2 \to \R$ be a $L$-pseudo-Lipschitz test functions satisfying \eqref{eq:pseudoLip}. For any $a \in [q]$, the following limit holds almost surely
\[
\hspace{-9cm}\lim_{N \to \infty} \frac{1}{|C_a^N|} \sum_{i \in C_a^N} \phi( x_i^t, x_i^\star ) = \mathbb{E}_{x_{0}^{\star}, Z} \phi( \mu_{a}^{t}x_{0}^{\star} + \sigma_{a}^{t}Z, x_0^\star ) 
\]
where $Z$ is a standard Gaussian independent from all other variables.
\end{theorem}

\begin{remark}
    The notion of convergence under the pseudo-Lipschitz test functions induces a topology that is equivalent to the one generated by the $2$-Wasserstein topology \cite[Remark~7.18]{feng2022unifying}. We can weaken the second moment assumption on $\mathbb{P}_0$ to finite $k$th moment, but the induced topology will then change to the $k$-Wasserstein topology, see \cite[Theorem~1]{8205391}.
\end{remark}

Even though the theoretical result above applies in the high-dimensional limit, numerical simulations show that even for medium-sized $N$ (around $500$), the behaviour of the iterates is well described by the state evolution parameters. Through the state evolution equations \eqref{eq:state} we are able to track the iterates of the AMP iteration with just two vectors of parameters obeying the state evolution recursion: the overlap with the true signal $(\mu_a^{t})_{a \in [q]}$ and its variance $(\sigma^{t}_a)_{a \in [q]}$. 
For the inhomogeneous AMP \eqref{eqn: low-rank inhomogenous} iteration we obtain the following necessary and sufficient condition for the overlaps of a fixed point of the iteration:

\begin{theorem}[Bayes-Optimal fixed point]\label{theorem: AMP fixed point}
Assume AMP satisfies Assumption~\ref{ass:limit} and let the denoising functions be the Bayes ones \eqref{eqn: Bayes-optimal functions}. Then the overlaps $\bm{\mu} = (\mu_a)_{a \in [q]}$ in \eqref{eq:state} satisfy the following fixed point equation
    \begin{equation} 
	\mu_{b} = \sum_{a \in[q]} \frac{c_{a}}{\tilde{\boldsymbol{\Delta}}_{ab}}\mathbb{E}_{x_{0}^{\star}, Z}[x_{0}^{\star}\mathbb{E}_{posterior}[x_{0}^{\star}|\mu_{a}x_{0}^{\star} + \sqrt{\mu_{a}}Z]].
\end{equation}

\end{theorem}

\begin{remark}
	The state evolution fixed point equation above coincides with the fixed point equation satisfied by the Bayes optimal estimator in \cite[Equation~2.14]{AJFL_inhomo}.
\end{remark}

\begin{figure}[t]
	\begin{center}
        \includegraphics[width = 10cm]{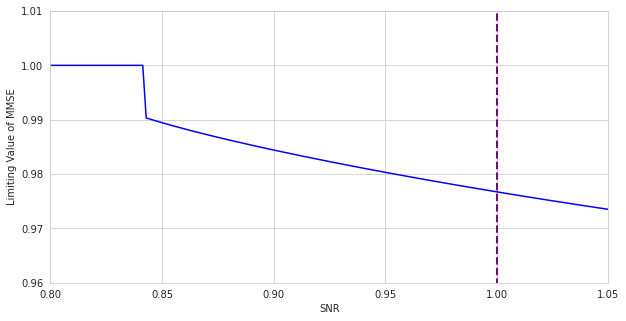}
		\caption{The information-theoretic optimal mean squared error for a sparse prior. The AMP detectability transition occurs at $\lambda(\bm{\Delta}) = 1$ in Section~\ref{sec:linear}, which will yield a statistical-to-computational gap. This is in contrast to the continuous phase transition in Figure~\ref{fig:dense_general}, where the AMP phase transition agrees with the optimal one}
		\label{fig:optimal}
	\end{center}
     \vspace{-0.5cm}
\end{figure}

\paragraph{\textbf{A spectral method}}
The spectrums of matrices with variance profiles are difficult to analyze because standard tools to compute the BBP transition result in complicated systems of equations. Given the matrix $\bm{Y}$ defined in \eqref{eqn: low-rank inhomogenous} we consider the transformed matrix
\begin{equation}\label{eq:transformedYgenvariance}
	\tilde {\bm{Y}} :=  \frac{\mathbb{E}_{x_0^\star} [(x_0^\star)^2]}{\sqrt{N}\bm{\Delta}} \odot \bm{Y} - \mathbb{E}_{x_0^\star} [(x_0^\star)^2]^2 \operatorname{diag}\left(\frac{1}{\bm{\Delta}} \begin{bmatrix}
		1 \\
		\vdots \\
		1
	\end{bmatrix} \right) .
\end{equation}
Using AMP tools, we are able to analyze the spectral method. In particular, we are able to recover the phase transition for the top eigenvalue of spiked matrices with covariance profiles through the inhomogeneous AMP. Let $\bm{c} = (c_a)_{a \in [q]}$. We define the inhomogeneous signal-to-noise (SNR) ratio of such a model by
\begin{equation} \label{eq:SNR}
{\rm SNR}({\bm \Delta}) := \lambda(\bm{\Delta}) =  \mathbb{E}_{x_0^\star} [(x_0^\star)^2]^2 \left\| \operatorname{diag}(\sqrt{\bm{c}}) \frac{1}{\tilde{\bm{\Delta}}} \operatorname{diag}(\sqrt{\bm{c}}) \right\|_{op}.
\end{equation}

\begin{conjecture}\label{conjecture:BBP}
The top eigenvalue of $\tilde {\bm{Y}}$ separates from the bulk if and only if the signal to noise ratio $\lambda(\bm{\Delta}) > 1$. In particular, if $\bm{\hat{x}}$ is the top eigenvector of $\tilde{\bm{Y}}$  then 
\[
\lim_{N \to \infty} \frac{|\bm{\hat{x}} \cdot \bm{x}^\star|}{ \|\bm{\hat{x}} \| \|\bm{x}^\star \| }  = 0
\]
if and only if $\lambda(\bm{\Delta})  < 1$.
\end{conjecture}
This matches precisely the recovery transition in \cite[Lemma~2.15 Part (b)]{AJFL_inhomo}. In this paper, we rigorously show that with ${\rm SNR}({\bm \Delta}) < 1$ our proposed spectral method fails to recover the signal. We conjecture that this is the sharp transition for spectral methods.

We postpone a full mathematical analysis of this spectral method for future studies. However, we provide indications for validity of the result by using the linear AMP formalism. The connection between the two is a standard phenomenon \cite{saade2014spectral,mondelli2018fundamental,maillard2022construction}. We illustrate the eigenvalue BBP-like transition in  Fig.\ref{fig:spectrum}.


\begin{figure}[t]
	\begin{center}
        \includegraphics[width=7cm]{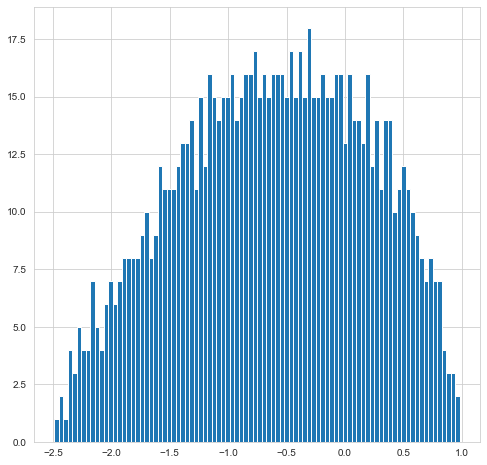}
        		\includegraphics[width=7cm]{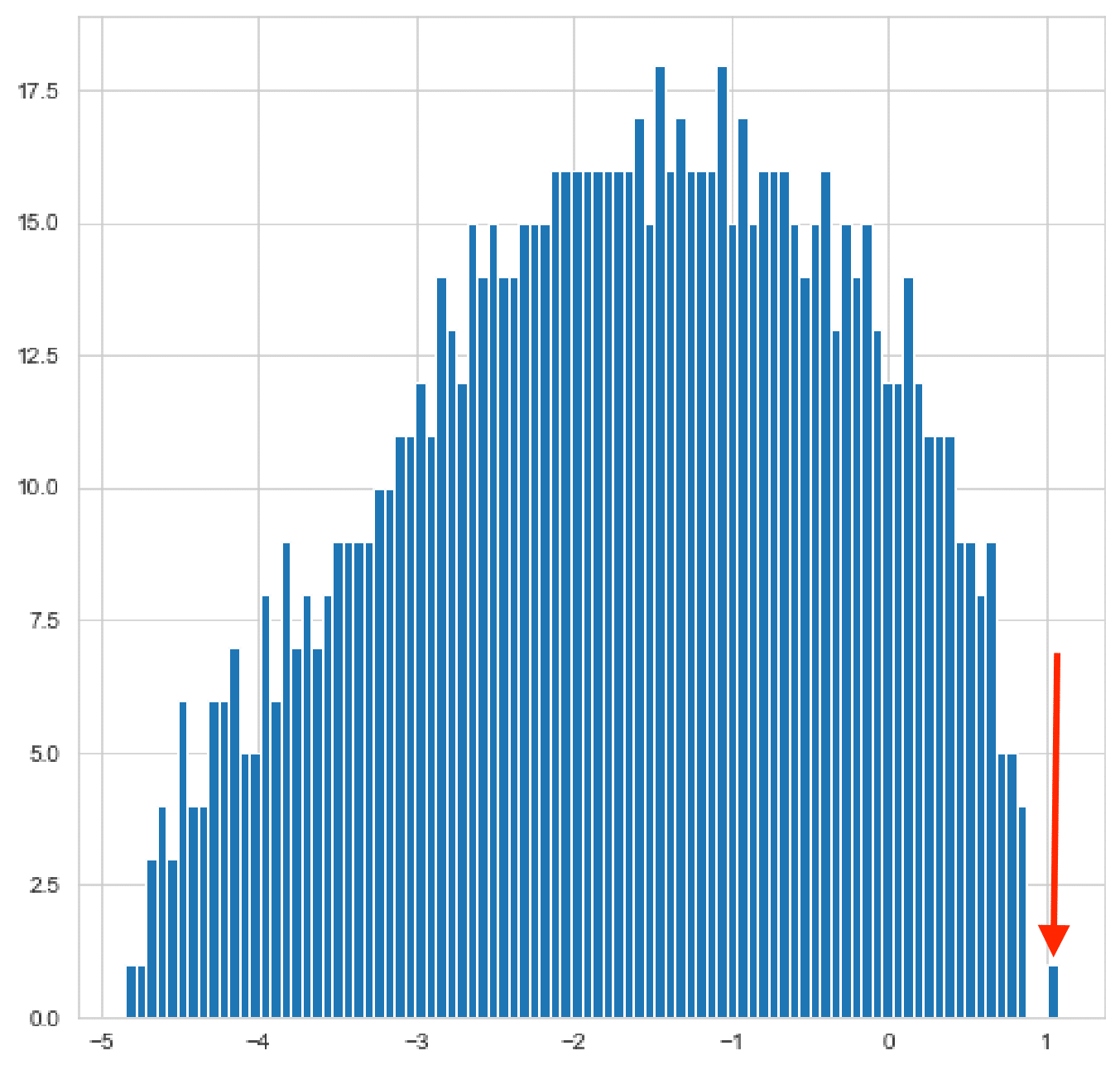}

		\caption{Illustration of the spectrum of $\bm{\tilde Y} \in \R^{10^3 \times 10^3}$ evaluated at noise profiles with {\rm snr} $\lambda(\bm{\Delta}) = 0.7$ (left, before the transition) and $1.8$ (right, after the transition), with the outlying eigenvector correlated with the spike arises at eigenvalue one.}
  \label{fig:spectrum}
	\end{center}
\end{figure}

\paragraph{\textbf{Statistical to computational gaps}}As the linear AMP transition arises at $\lambda\!>\!1$, the linear stability analysis of AMP initialized close to a trivial fixed point will recover an identical transition. As in the homogeneous case, the inhomogeneous problem is thus algorithmically tractable only for $\lambda\!>\!1$. However, it was shown that, for sparse enough priors, the Bayes estimate (that is possibly NP-hard) can achieve a positive correlation for $\lambda\!<\!1$. This illustrates the statistical-to-computational gap as in e.g. \cite{bandeira2018notes,barbier2018rank,aubin2019spiked,lelargemiolanematrixestimation,celentano2020estimation}. In this situation, the spectral method described in this work should thus be optimal. Interestingly, this is {\it not} the case of the standard PCA analysis based on the matrix $\bm{Y}$, which fails to achieve  a transition at $\lambda\!=\!1$ \cite[Proposition~2.18]{AJFL_inhomo}.



\section{The inhomogeneous AMP algorithm}\label{sec:AMP}

In this section, we derive the formula for the inhomogeneous AMP iteration \eqref{eqn: AMP spike}. We first recall the general matrix framework of AMP from \cite{8205391}:

\paragraph{\textbf{Matrix AMP}}\label{sec:matrixAMP}
	In the matrix setting an AMP algorithm operates on the vector space $\mathcal{V}_{q, N} \equiv\left(\mathbb{R}^{q}\right)^{N} \simeq \mathbb{R}^{N \times q}$. Each element of $\bv = (v_1, \dots, v_N) \in \mathcal{V}_{q, N}$ will be regarded as $N$ - vector with entries $v_i \in \mathbb{R}^{q}$.
	\begin{definition}[AMP]\label{def:AMP}
		A matrix AMP acting on this space is represented by $(\boldsymbol{A}, \mathcal{F}, \boldsymbol{v}^{0})$, where:
		\begin{enumerate}
			\item $\boldsymbol{A} = \boldsymbol{G} + \boldsymbol{G}^\trans$, where $\boldsymbol{G} \in \R^{N \times N}$ has iid entries $G_{ij} \sim N(0,\frac{1}{2})$. 
			\item $\mathcal{F}$ is a family of $N$ Lipschitz functions $f_t^i: \R^q \mapsto \R^q$ indexed by time $t$. The family $\mathcal{F}$ encodes a function $f_{t}: \mathcal{V}_{q, N} \rightarrow \mathcal{V}_{q, N}$ that acts separately on each coordinate $v_{j} \in \mathbb{R}^{q}$,
			\begin{equation}
				f_{t}(\boldsymbol{v}) = (
				f_{t}^{1}(v_{1}), \dots,     f_{t}^{N}(v_{N})
				\in \mathcal{V}_{q, N}.
			\end{equation}
			\item $\boldsymbol{v}^{0} \in \mathcal{V}_{q, N}$ is a starting condition.
		\end{enumerate}
	\end{definition}

The algorithm itself is a sequence of iterates generated by: 
\begin{equation} \label{eqn: Matrix MAMP}
	\boldsymbol{v}^{t+1}= \frac{\boldsymbol{A}}{\sqrt{N}} f_{t}\left(\boldsymbol{v}^{t}\right) - f_{t-1}\left(\boldsymbol{v}^{t-1}\right)\bm{\mathrm{B}}_{t}^{T}
\end{equation}
where $\bm{\mathrm{B}}_{t}$ is the $q \times q$ Onsager matrix given by
\begin{equation} \label{eqn: Onsager matrix}
	\bm{\mathrm{B}}_{t} = \frac{1}{N} \sum_{j = 1}^{N} \partial f_{t}^{j} (\boldsymbol{x}_{j}^{t}).
\end{equation}
where $\partial f_{t}^{j}$ denotes the Jacobian matrix of $f_t^j$.  The limiting properties of the AMP sequences are well known and can be found in \cite[Theorem~1]{8205391}.
\paragraph{\textbf{The inhomogeneous AMP}} \label{sec:inhomoAMP}
We now define an inhomogeneous AMP iteration which takes into account the block-constant structure of the noise:
\begin{definition}\label{def:inhomoAMP}
	An inhomogeneous AMP on $\mathcal{V}_{1,N} = \R^N$ is represented by $( \boldsymbol{A}, \mathcal{F}, \boldsymbol{x}^{0},\bm{\Delta})$, where the terms $\boldsymbol{A}$, $\mathcal{F}$, $\boldsymbol{x}^{0}$  are defined in Definition~\ref{def:AMP}  and $\bm{\Delta}$ is the $N \times N$ variance profile encoded by $\tilde{ \bm{\Delta}} \in \R^{q \times q}$ and grouping $g: [N] \to [q]$ defined by \eqref{eq:varprofile}. We further assume that the family of functions $\mathcal{F}$ is encoded by functions $f_t^a: \R \mapsto \R $ for $a \in [q]$ which define the group dependent function
    \begin{equation}
				f_{t}(\boldsymbol{x}) = (
				f_{t}^{g(1)}(x_{1}), \dots,     f_{t}^{g(N)}(x_{N}) )
				\in  \R^N.
			\end{equation}
\end{definition}

The sequence of iterates $\boldsymbol{x}^{t} \in \mathbb{R}^{N}$ of the $( \boldsymbol{A}, \mathcal{F}, \boldsymbol{v}^{0},\bm{\Delta})$ are defined as follows:
\begin{equation} \label{eqn: version}
	\boldsymbol{x}^{t+1}= \left( \frac{1}{\sqrt{N} \sqrt{\boldsymbol{\Delta}}} \odot \boldsymbol{A} \right) f_{t} \left(\boldsymbol{x}^{t}\right)-\mathrm{\boldsymbol{b}}_{t} \odot f_{t-1} \left(\boldsymbol{x}^{t-1}\right),
\end{equation}
where $\frac{1}{\sqrt{\bm{\Delta}}}$ is the Hadamard inverse square root of the noise,  and the Onsager term $\boldsymbol{\mathrm{b}}_{t}$ has the following form
\begin{equation} \label{eqn: Onsager vector}
	\boldsymbol{\mathrm{b}}_{t} = \frac{1}{N} \begin{pmatrix}
		\frac{1}{\Delta_{11}} (f_{t}^{g(1)})^{\prime} (x_{1}^{t}) &+&  \ldots &+& \frac{1}{\Delta_{1N}} (f_{t}^{g(N)})^{\prime} (x_{N}^{t})\\
		\vdots && \vdots && \vdots \\
		\frac{1}{\Delta_{N1}} (f_{t}^{g(1)})^{\prime} (x_{1}^{t}) &+&  \ldots &+& \frac{1}{\Delta_{NN}} (f_{t}^{g(N)})^{\prime} (x_{N}^{t})
	\end{pmatrix} = \frac{1}{N\boldsymbol{\Delta}} f_{t}^{'} (x^{t}) \in \R^N	.
\end{equation}
In order to track the iterates of the recursion \eqref{eqn: version} we reduce this recursion to the matrix setting with an embedding of the inhomogeneous AMP into the matrix AMP. 

\paragraph{\textbf{State evolution of the inhomogeneous AMP}} Through a continuous embedding, we will reduce our inhomogeneous AMP to the matrix AMP framework, and recover the state evolution of the inhomogeneous AMP. We define the diagonal matrix operator $\operatorname{blockdiag}: \R^N \mapsto \mathcal{V}_{q, N}$ which outputs a block diagonal matrix according to the block structure of our discretization of $[N]$:
\[
\operatorname{blockdiag}(\bv) = \bm{M} \quad \text{where} \quad M_{ij} = \begin{cases}
v_{j} & g(j) = i
\\    0 & \text{otherwise} .
\end{cases}
\]
Likewise, we define the projection operator $\operatorname{blockproj}: \mathcal{V}_{q, N} \mapsto \R^N$ which extracts a vector of size $N$ from a $N \times q$ according to the block structure of $[N]$ by
$$\bp(\boldsymbol{M}) = (M_{ig(i)})_{i \leq N} \in \R^N.$$
Under these changes of variables, we define
\[
\boldsymbol{r}^{t} = \bd(\boldsymbol{x}^{t}) \in \mathcal{V}_{q, N} \quad \text{ for $t \geq 0$}
\]
and $\tilde{f_{t}}$: $(\mathbb{R}^{q})^{N} \mapsto (\mathbb{R}^{q})^{N}$ by
\begin{equation} 
	\left(\tilde{f}_{t}(\boldsymbol{r}^{t})\right)_{ij} = \frac{1}{\sqrt{\tilde{\Delta}_{g(i)j}}}f_{t}^{g(i)}(x_{i}^{t}) \qquad \text{for $i,j \in [N] \times [q]$.}
\end{equation}
We encode the family of functions $\tilde f_t$ by $\tilde{\mathcal{F}}(\bm{\Delta})$. 

\begin{lemma}
Let $\bx^t$ be iterates from the AMP $(\boldsymbol{A}, \mathcal{F}, \boldsymbol{v}^{0},\bm{\Delta})$. Then the iterates $\bm{r}^t :=\bd(\bx^t)$ follow the generalized matrix AMP $(\bm{A}, \tilde{\mathcal{F}}(\bm{\Delta}), \bm{r}^0)$.  
\end{lemma}
\begin{proof} 
 We will show that the projection of the iterates $\bm{r}^t$ from $(\bm{A}, \tilde{\mathcal{F}}(\bm{\Delta}), \bm{r}^0)$ are the iterates from  $(\boldsymbol{A}, \mathcal{F}, \boldsymbol{v}^{0},\bm{\Delta})$. It is easy to check that
	\begin{equation}
		\frac{\boldsymbol{A}}{\sqrt{N}  \sqrt{\boldsymbol{\Delta}}} f_{t}(\boldsymbol{x}^{t}) = \bp\left(\frac{\boldsymbol{A}}{\sqrt{N} }\tilde{f}_{t}(\boldsymbol{r}^{t})\right) .
	\end{equation}
Next, notice that Jacobian is given by
$$\partial \tilde{f}_{t}^{i}(\bm{r}_i^t) =
\left[\begin{array}{ccccc}
	0 & \cdots & \frac{1}{\sqrt{\tilde{\Delta}_{g(i)1}}}(f^{g(i)}_{t})^{\prime} (x_{i}^{t}) & \cdots & 0 \\
	\vdots & \ddots & \vdots & \ddots &\vdots \\
	0 & \cdots & \frac{1}{\sqrt{\tilde{\Delta}_{g(i)q}}}(f^{g(i)}_{t})^{\prime} (x_{i}^{t}) &\cdots & 0
\end{array}\right],$$
which is a matrix where only the column number $g(i)$ has non-zero elements.  Applying \eqref{eqn: Onsager matrix}, we thus for $a, b \in [q] \times [q]$
\begin{equation} \label{eqn: Onsager term change of variables}
	\left(\boldsymbol{B}_{t}\right)_{ab} = \frac{1}{N\sqrt{\tilde{\boldsymbol{\Delta}}_{ab}}} \sum_{i: g(i) = a}(f_{t}^{a})^{\prime}(x_{i}^{t}) .
\end{equation}
It follows that
\[
\mathrm{\boldsymbol{b}}_{t} \odot f_{t - 1}(\boldsymbol{x}^{t - 1}) = \bp\left(\tilde{f}_{t - 1}(\boldsymbol{r}^{t - 1})\boldsymbol{B}_{t}^{T}\right). 
\]
\end{proof}

As a consequence, the inhomogeneous state evolution equations in Theorem~\ref{theorem: AMP inhomogeneous} follow immediately from the state evolution equations of $(\bm{A}, \tilde{\mathcal{F}}(\bm{\Delta}), \bm{r}^0)$ discussed in \cite[Section~2.1]{8205391}. 
This follows from the observation that given the law of $\bm{r}^t$ in the high dimensional limit, the law of $\bm{x}^t = \bp( \bm{r}^t )$ is straightforward to compute. We define
\begin{equation}\label{eq:sigma_centered}
(\sigma_b^{t + 1})^2 :=	\sum_{a=1}^{q} \frac{c_{a}}{\tilde{\boldsymbol{\Delta}}_{ab}}\mathbb{E}\left[ (f^b_{t}(Z_{b}^{t}))^{2}\right] .     
\end{equation}
We will show that the distribution of the iterate $\bx_i^t$ is asymptotically normal with mean $0$ and variance $(\sigma_{g(i)}^t)^2$.  

\begin{lemma}[Behavior of the AMP iterates in the inhomogeneous setting with no spike] \label{lem:state_evoluton_centered}
    Suppose that Assumption~\ref{ass:limit} holds, and that $\mathbb{P}_0$ has a bounded second moment. Let $\phi: \R^2 \to \R$ be a $L$-pseudo-Lipschitz test functions satisfying \eqref{eq:pseudoLip}. For any $a \in [q]$, then the following limit holds almost surely
\[
\lim_{N \to \infty} \frac{1}{|C_a^N|} \sum_{i \in C_a^N} \phi( x_i^t,  x_i^\star ) = \mathbb{E}_{x_{0}^{\star}, Z} \phi(  \sigma_{a}^{t}Z,  x_0^\star ) 
\]
where $Z$ is an independent standard  Gaussian.
\end{lemma}
\begin{proof}
    In matrix-AMP \cite[Th.~1]{8205391},  the marginals of the iterates $\bm{r^t}$ from  $(\bm{A}, \tilde{\mathcal{F}}(\bm{\Delta}), \bm{r}^0)$ are approximately Gaussian and encoded by the positive definite matrices
\begin{eqnarray}
	\widehat{\boldsymbol{\Sigma}}_{a}^{t}=\mathbb{E}\left[\tilde{f}_{t}^{i}\left(\boldsymbol{Z}^{t}\right) \tilde{f}_{t}^{i}\left(\boldsymbol{Z}^{t}\right)^{\top}\right] \qquad \text{for all $i \in C_a^N$},\\
 {\rm with}\,\,\boldsymbol{Z}^{t} \sim N(0, \boldsymbol{\Sigma}^{t})\,\, \rm{and}\,\,
\label{eqn: linear_comb}
	\boldsymbol{\Sigma}^{t + 1}=\sum_{a=1}^{q} c_{a} \widehat{\boldsymbol{\Sigma}}_{a}^{t}.
\end{eqnarray}
We now show that $a \in [q]$ and $i \in C_a^N$, $\widehat{\boldsymbol{\Sigma}}_{a}^{t}$ depends only on $(\sigma_{a}^{t})^{2}  = \boldsymbol{\Sigma}_{aa}^{t}$. Indeed, by the definition of $\tilde{f}_{t}^{i}$ we have 
\begin{equation} \label {eqn: sigma_hat}
	\widehat{\boldsymbol{\Sigma}}_{a}^{t} (k, l) = \mathbb{E}\left[(\tilde{f}_{t}^{i}\left(\boldsymbol{Z}^{t}\right))_{k} (\tilde{f}_{t}^{i}\left(\boldsymbol{Z}^{t}\right))_{l}\right] = \mathbb{E}\left[ \frac{1}{\sqrt{\tilde{\boldsymbol{\Delta}}_{ak}}}  \frac{1}{\sqrt{\tilde{\boldsymbol{\Delta}}_{al}}} (f_{t}^{a}(Z_{a}^{t}))^{2}\right],
\end{equation} 
where $Z_{a}^{t} \sim N(0, (\sigma_{a}^{t})^{2})$ is the $a$th component of the Gaussian vector $\boldsymbol{Z}_{t}$. The key observation here is that by construction our function $\tilde{f}_{t}^{i}, \mathbb{R}^{q} \mapsto \mathbb{R}^{q}$ depends only on the ith component $Z_{i}^{t}$ of the Gaussian vector $\boldsymbol{Z}^{t}$. 
To characterize the limiting distribution of $\boldsymbol{x}^{t} = \bp(\boldsymbol{r}^{t})$, we only need to keep track of the variances $(\sigma_{j}^{t})^{2}, j \in [N]$. Using \eqref{eqn: sigma_hat}, for a given $a \in [q]$ and $i \in C_{a}^{N}$ we write 
\begin{equation} 
	\widehat{\boldsymbol{\Sigma}}_{a}^{t} (g(j), g(j)) = \mathbb{E}\left[(\tilde{f}_{t}^{i}\left(\boldsymbol{Z}^{t}\right))_{g(j)} (\tilde{f}_{t}^{i}\left(\boldsymbol{Z}^{t}\right))_{g(j)}\right]  = \frac{1}{\tilde{\boldsymbol{\Delta}}_{ag(j)}}\mathbb{E}\left[ (f_{t}^{i}(Z_{i}^{t}))^{2}\right].
\end{equation} 
Finally, with \eqref{eqn: linear_comb} we get that for any $b \in [q]$ and any $j \in C_b^N$, 
using $Z_{a}^{t} \sim N(0, (\sigma_{a}^{t})^{2})$,
\begin{equation} \label{eqn: state_evolution}
	(\sigma_{j}^{t + 1})^{2} = (\sigma_b^{t + 1})^2  = \boldsymbol{\Sigma}_{bb}^{t + 1} = \sum_{a=1}^{q} c_{a} \widehat{\boldsymbol{\Sigma}}_{a}^{t}(g(j), g(j)) = \sum_{a=1}^{q} \frac{c_{a}}{\tilde{\boldsymbol{\Delta}}_{ag(j)}}\mathbb{E}\left[ (f_{t}^{a}(Z_{a}^{t}))^{2}\right]
\end{equation}
\end{proof}

\paragraph{\textbf{The inhomogeneous spiked Wigner model in the light of the AMP approach}}\label{sec:spiked}
We now generalize the state evolution equations from Lemma~\ref{lem:state_evoluton_centered} to spiked matrices with an inhomogenous noise profile as was stated in Theorem~\ref{theorem: AMP inhomogeneous}. This reduction via a change of variables is standard, see for example \cite[Lemma 4.4]{DBLP:journals/corr/DeshpandeAM15}. Remember that in the inhomogeneous version of the spiked Wigner model we observe the signal $\boldsymbol{x}^{\star}$ through an inhomogeneous channel:

\begin{equation} \label{eq:spikedWigner}
	\boldsymbol{Y} = \sqrt{\frac{1}{N}}\boldsymbol{x}^{\star}(\boldsymbol{x}^{\star})^{T} + \boldsymbol{A} \odot \sqrt{\boldsymbol{\Delta}}.
\end{equation}
Our AMP algorithm is defined with the following recursion:
\begin{equation} 
	\boldsymbol{x}^{t+1}= \bigg( \frac{1}{\sqrt{N}\boldsymbol{\Delta}} \odot \bm{Y} \bigg) f_{t}\left(\boldsymbol{x}^{t}\right)-\boldsymbol{\mathrm{b}}_{t} \odot f_{t-1}\left(\boldsymbol{x}^{t-1}\right)
\end{equation}
where $\boldsymbol{\mathrm{b}}_{t} = \frac{1}{\boldsymbol{\Delta}} f_{t}^{'} (\boldsymbol{x}^{t})$ and $f$ is encoded by the family of functions in Definition~\ref{def:inhomoAMP}. The main difference in contrast to the iteration \eqref{eqn: version} is that our data matrix $\bm{Y}$ is no longer a centered matrix, while $\frac{1}{\sqrt{\bm{\Delta}}} \odot \bm{A}$ is. We would like to reduce \eqref{eqn: AMP spike} to an iteration of the form \eqref{eqn: version} with respect to a different parameter $\bm{s}^t$ which is uniquely determined by $\bm{x}^t$
\[
\boldsymbol{s}^{t+1}=\bigg(\frac{1}{\sqrt{N}  \sqrt{\boldsymbol{\Delta}}} \odot \boldsymbol{A} \bigg) f_{t}\left(\boldsymbol{s}^{t}\right)-\boldsymbol{\mathrm{b}}_{t} \odot f_{t-1}\left(\boldsymbol{s}^{t-1}\right).
\]
Doing so will allow us to recover the limiting laws of the iterates from Lemma~\ref{lem:state_evoluton_centered}. This is done via a standard change of variables to recenter $\bm{Y}$. We will sketch the argument in this section, and defer the full proof of Theorem~\ref{theorem: AMP inhomogeneous} to the Appendix~\ref{app:theostate}. 

To simplify notation, we will often denote $f_{t}(\boldsymbol{x}^{t}) := \boldsymbol{\hat{x}}^{t}$. 
We proceed following the approach of \cite[Lemma~4.4]{DBLP:journals/corr/DeshpandeAM15}. We will rewrite \eqref{eqn: AMP spike} using the definition of Y to get
\begin{equation}\label{eq:xiterates1}
	\begin{aligned}
		\boldsymbol{x}^{t+1} &= \bigg( \frac{1}{\sqrt{N}\boldsymbol{\Delta}} \odot \boldsymbol{Y} \bigg) f_{t}\left(\boldsymbol{x}^{t}\right)-\boldsymbol{\mathrm{b}}_{t} \odot f_{t-1}\left(\boldsymbol{x}^{t-1}\right) \\
		&= \bigg( \frac{1}{\sqrt{N}\boldsymbol{\Delta}} \odot \boldsymbol{Y} \bigg) \boldsymbol{\hat{x}}^{t}-\boldsymbol{\mathrm{b}}_{t} \odot \boldsymbol{\hat{x}}^{t-1}\\
		&=\bigg( \frac{1}{N \boldsymbol{\Delta}} \odot \boldsymbol{x^{\star}}\boldsymbol{(x^{\star})^{T}} \bigg) \boldsymbol{\hat{x}}^{t} + \bigg( \frac{1}{\sqrt{N} \sqrt{\boldsymbol{\Delta}}} \odot \bm{A} \bigg) \boldsymbol{\hat{x}}^{t}-\boldsymbol{\mathrm{b}}_{t} \odot \boldsymbol{\hat{x}}^{t-1}.
	\end{aligned}
\end{equation}
If indices are independent, then by the strong law of large numbers one would expect that
\begin{equation}\label{eqn: SLLN}
	\left(\bigg( \frac{1}{N \boldsymbol{\Delta}} \odot \boldsymbol{x^{\star}}\boldsymbol{(x^{\star})^{T}} \bigg) \boldsymbol{\hat{x}}^{t} \right)_{j} = x_{j}^{\star} \sum_{a \in [q]} \sum_{i \in C_{a}^{N}} \frac{1}{N} \frac{x_{i}^{\star}\hat{x}_{i}^{t}}{\boldsymbol{\Delta}_{ji}}  \rightarrow x_{j}^{\star} \sum_{a \in [q]} \frac{c_{a}}{\boldsymbol{\Delta}_{ji_{a}}} \mathbb{E}[x_{0}^{\star}\hat{x}_{i_{a}}^{t}],
\end{equation} 
where $i_{a}$ is some index belonging to the group $C_{a}^{N}$ and $x_{0}^{\star}$ is a random variable distributed according to the prior distribution $\mathbb{P}_{0}$.  For $b \in [q]$ and $i \in C_b^N$, we define the block overlap $\mu_{b}^{t}$ using the recursion
\begin{equation} \label{eqn: mu}
	\mu^{t + 1}_{i} = \mu^{t + 1}_{b} =  \sum_{a \in [q]} \frac{c_{a}}{\tilde{\boldsymbol{\Delta}}_{ab}}  \mathbb{E}_{x_{0}^{\star}, Z}[x_{0}^{\star} f_{t}^{a}\left(\mu^{t}_{a}x_{0}^{\star} + \sigma^{t}_{a}Z\right)],
\end{equation}
where $Z$ is a standard Gaussian random variable independent from all others sources of randomness. Notice that \eqref{eqn: mu} is precisely the asymptotic behavior of the summation appearing in \eqref{eqn: SLLN} by Lemma~\ref{lem:state_evoluton_centered}, which is how we control \eqref{eqn: SLLN} in the rigorous proof.

We now make a change of variables and track the iterates
\begin{equation}\label{eq:s_iterates}
    \boldsymbol{s}^{0} = \boldsymbol{x}^{0} - \boldsymbol{\mu}^{0} \odot \boldsymbol{x}^{\star} \qquad \boldsymbol{s}^{t} = \boldsymbol{x}^{t} - \boldsymbol{\mu}^{t} \odot \boldsymbol{x}^{\star}, \quad t \geq 1
\end{equation}
where $\boldsymbol{\mu}^{0}$ is the vector of block overlaps of the initial condition $\boldsymbol{x}^{0}$ with the truth. We reduced the \eqref{eqn: AMP spike} iteration to the following iteration in which we easily recognize a version of \eqref{eqn: version}:
\begin{equation} \label{eqn: AMP change spike}
	\boldsymbol{s}^{t+1}= \left( \frac{1}{\sqrt{N} \sqrt{\boldsymbol{\Delta}}}  \odot \boldsymbol{A}\right) f_{t} \left(\boldsymbol{s}^{t} + \boldsymbol{\mu^{t}} \odot \boldsymbol{x}^{\star}\right)-\bm{\mathrm{b}}_{t} \odot f_{t-1} \left(\boldsymbol{s}^{t-1} + \boldsymbol{\mu^{t - 1}} \odot \boldsymbol{x}^{\star}\right)
\end{equation}
with the initial condition $\boldsymbol{s}^{0} = \boldsymbol{x}^{0} - \boldsymbol{\mu^{0}} \odot \boldsymbol{x}^{\star}$ and the Onsager term taken from \eqref{eqn: Onsager vector} is given by
\begin{equation}
	\bm{\mathrm{b}}_{t} = \frac{1}{\boldsymbol{\Delta}} f_{t}^{'} (\boldsymbol{s}^{t} + \boldsymbol{\mu}^{t} \odot \boldsymbol{x}^{\star}).
\end{equation}

Using Lemma~\ref{lem:state_evoluton_centered}, we can recover the asymptotic behavior of the iterates $\bm{x}^t$ given in \eqref{eqn: AMP spike} by computing the iterates $\boldsymbol{s}^{t} + \boldsymbol{\mu}_{t} \odot \boldsymbol{x}^{\star}$ where $\boldsymbol{s}^{t}$ follows \eqref{eqn: AMP change spike} and $\bm{\mu}_t$ satisfies \eqref{eqn: mu}. From this reduction we obtain the following  state evolution equations describing the behaviour of \eqref{eqn: AMP spike}:

\begin{enumerate}
	\item $x_{j}^{t} \approxeq \mu^{t}_{g(j)}x_{0}^{\star} + \sigma_{g(j)}^{t}Z$ for $j \in [N]$, where $Z \sim \mathcal{N}(0, 1)$
	\item $\mu^{t + 1}_{b} = \sum_{a \in [q]} \frac{c_{a}}{\tilde{\Delta}_{ab}} \mathbb{E}_{x_{0}^{\star}, Z}[x_{0}^{\star} f_{t}^{a}\left(\mu^{t}_{a}x_{0}^{\star} + \sigma^{t}_{a}Z\right)]$ with $x_{0}^{\star} \sim \mathbb{P}_{0}, Z \sim \mathcal{N} (0, 1)$
	\item $(\sigma_{b}^{t + 1})^{2} = \sum_{a=1}^{q} \frac{c_{a}} {\tilde{\Delta}_{ab}}\mathbb{E}_{x_{0}^{\star}, Z}\left[ (f_{t}^{a}(\mu^{t}_{a}x_{0}^{\star} +  \sigma_{a}^{t}Z))^{2}\right]$ with $x_{0}^{\star} \sim \mathbb{P}_{0}$, $Z \sim \mathcal{N} (0, 1).$
\end{enumerate}

This (informally) characterizes the limiting distribution of the state evolution of the iterates from the inhomogeneous AMP stated in Theorem~\ref{theorem: AMP inhomogeneous}. The main technical difficulty is the equality in \eqref{eqn: SLLN} and \eqref{eqn: mu} already uses the asymptotic distribution of the overlaps at finite $N$. This technical difficulty is dealt with in the full proof of Theorem~\ref{theorem: AMP inhomogeneous} in Appendix~\ref{app:theostate}. 

\paragraph{\textbf{Fixed-point equation of state evolution in the Bayes-optimal setting}}\label{sec:Bayes}

Suppose that we know the prior distribution $\mathbb{P}_{0}$ of $x_{0}^{\star}$. The Bayes-optimal choice for the denoising functions $f_{t}^{j}, j \in [N]$ is simply the expectation of $x_{0}^{\star}$ with respect to the posterior distribution,  
\begin{equation} \label{eqn: Bayes-optimal functions}
	f_{t}^{j} (r) = f_{t}^{g(j)} (r) = \mathbb{E}_{posterior}[x_{0}^{\star}| \mu^{t}_{g(j)}x_{0}^{\star} + \sigma_{g(j)}^{t}Z = r].
\end{equation}
Under this Bayes-optimal setting, we can simplify the equations obtained in the previous section and see that AMP estimator is indeed an optimal one by studying its fixed point.
\begin{proof}[Theorem~\ref{theorem: AMP fixed point}]
    For this choice of  $f_{t}^{j}$ the Nishimori identity (see for example \cite[Proposition~16]{lelargemiolanematrixestimation}) states that for $a \in [q]$ and $j \in C_a^N$,
\begin{equation}
	\tilde{\mu}_{a}^{t}:= \mathbb{E}_{x_{0}^{\star}, Z}[x_{0}^{\star}f_{t}^{a}(\mu^{t}_{a}x_{0}^{\star} + \sigma_{a}^{t}Z)] = \mathbb{E}\left[ (f_{t}^{a}(\mu^{t}_{a}x_{0}^{\star} +  \sigma_{a}^{t}Z))^{2}\right].
\end{equation}
In this setting, the state evolution equations from Theorem~\ref{theorem: AMP inhomogeneous} initialized at $\bm{\mu}^0 = \bm{\sigma}^0 = 0$  reduce to
\begin{equation} \label{eqn: State evolution Bayes-optimal}
	\begin{cases}
		\tilde{\mu}_{a}^{t} = \mathbb{E}_{x_{0}^{\star}, Z}[x_{0}^{\star}f_{t}^{a}(\mu^{t}_{a}x_{0}^{\star} + \sigma_{a}^{t}Z)] \\ 
		\mu^{t + 1}_{b} = \sum_{a \in [q]} \frac{c_{a}}{\tilde{\Delta}_{ab}}\ \tilde{\mu}_{a}^{t}, b \in [q] \\
		(\sigma_{b}^{t + 1})^{2} = \sum_{a \in [q]} \frac{c_{a}}{\tilde{\Delta}_{ab}}\ \tilde{\mu}_{a}^{t}, b \in [q].
	\end{cases}
\end{equation}
Remarkably with the Bayes-optimal choice of the denoising functions we have that for $t \geq 1$ for each block $b \in [q]$, $\mu^{t + 1}_{b} = (\sigma_{b}^{t + 1})^{2}$. Therefore a necessary and sufficient condition for an estimator to be a fixed point of the state evolution is to simply have its overlaps $\mu^{t}_{b}$ unchanged by an iteration of the state evolution. This translates into the following equation for the overlaps $\mu_{b}, b \in [q]$ 
\begin{equation} \label{eqn: fixed point}
	\mu_{b} = \sum_{a \in[q]} \frac{c_{a}}{\tilde{\Delta}_{ab}}\mathbb{E}_{x_{0}^{\star}, Z}[x_{0}^{\star}\mathbb{E}_{posterior}[x_{0}^{\star}|\mu_{a}x_{0}^{\star} + \sqrt{\mu_{a}}Z]].
\end{equation}
The result of Theorem~\ref{theorem: AMP fixed point} now follows immediately. 
\end{proof}


\section{A spectral method adapted to the inhomogeneous spiked Wigner model}
\label{sec:spec}
In this section, we will describe how one can use the convergence of the inhomogeneous AMP with a simple choice of denoiser to recover the BBP transition of spiked Wigner matrices.

\paragraph{\textbf{From AMP to a spectral method}}\label{sec:linear}
Remarkably, AMP and the state evolution machinery associated with it can help us design a simple spectral algorithm that matches the information-theoretic phase transition \cite[Remark~2.16]{AJFL_inhomo}. Recall that Theorem~\ref{theorem: AMP inhomogeneous} does not require the denoising functions $f_{t}$ to be Bayes-optimal, but can be applied to any Lipschitz family of functions. In this section, we analyze the state evolution for the family of identity functions, $f_{t}(x) = x$.  By Remark~\ref{rem:scaling}, we can assume that the entries of $\bm{x}^\star$ have unit variance. With this choice of denoising functions the AMP iteration will simply become:
\begin{align} \label{eqn: AMP identity denosing}
	\bm{x}^{t + 1} = \left( \frac{1}{\sqrt{N}\bm{\Delta}} \odot \bm{Y} \right) \bm{x}^{t} - \bm{\mathrm{b}}_{t} \odot \bm{x}^{t - 1} \quad \text{ where } \quad
	\bm{\mathrm{b}}_{t} = \frac{1}{\bm{\Delta}} f_{t}^{'} = \frac{1}{\bm{\Delta}} \begin{bmatrix}
		1 \\
		\vdots \\
		1
	\end{bmatrix}.
\end{align}
If we denote $\bm{B}_{t} = \operatorname{diag}( \bm{\mathrm{b}}_{t})$, it is easy to see that the fixed point of this iteration yields
\begin{equation}
	\bm{x} = \left( \frac{1}{\sqrt{N}\bm{\Delta}} \odot \bm{Y} \right) \bm{x} - \bm{B}_{t} \bm{x}
\end{equation}
so any $\bm{x}$ fixed by the AMP iteration \eqref{eqn: AMP identity denosing} must be an eigenvector of the matrix 
\begin{equation}\label{eq:transformedY}
	\bm{\tilde{Y}} = \left( \frac{1}{\sqrt{N}\bm{\Delta}} \odot \bm{Y} \right) - \bm{B}_{t} = \left(\frac{1}{\sqrt{N}\bm{\Delta}} \odot \bm{Y} \right) - \operatorname{diag}\left(\frac{1}{\bm{\Delta}} \begin{bmatrix}
		1 \\
		\vdots \\
		1
	\end{bmatrix} \right) .
\end{equation}
A simple spectral method consists in taking the principal eigenvector (associated to the largest eigenvalue) of the matrix $\frac{\bm{Y}}{\sqrt{N}\bm{\Delta}} - \bm{B}_{t}$, and this linear AMP is a quick way to find such an eigenvector. 

\paragraph{\textbf{Analysis of the spectral method using state evolution}} It is expected that the spectral algorithm described above behaves as the AMP iteration with identity denoising functions around its fixed point. Therefore we can analyze this spectral algorithm using state evolution machinery for the AMP iteration. In the case of identity functions we have $f_{t}^{a}(x) = x$ for all $a \in [q]$, so 
\begin{align}
	\mathbb{E}_{x_{0}^{\star}, Z}[x_{0}^{\star}f_{t}^{a}(\mu^{t}_{a}x_{0}^{\star} + \sigma_{a}^{t}Z)] &=  \mathbb{E}_{x_{0}^{\star}, Z}[x_{0}^{\star}(\mu^{t}_{a}x_{0}^{\star} + \sigma_{a}^{t}Z)] =  \mu^{t}_{a} \\
	\mathbb{E}_{x_{0}^{\star}, Z}[(f_{t}^{a}(\mu^{t}_{a}x_{0}^{\star} + \sigma_{a}^{t}Z))^{2}] &= 
	\mathbb{E}_{x_{0}^{\star}, Z}[(\mu^{t}_{a}x_{0}^{\star} + \sigma_{a}^{t}Z)^{2}] = \ (\mu^{t}_{a})^{2} + (\sigma^{t}_{a})^{2} \label{eq:variancesigma}
\end{align}
which transforms state evolution equations \eqref{eq:state} into the following simple form: 
\begin{align}
	\mu^{t + 1}_{b} = \sum_{a \in [q]} \frac{c_{a}}{\Delta_{ba}}  \mu^{t}_{a} \quad \text{and} \quad
	(\sigma^{t + 1}_{b})^{2} = \sum_{a \in [q]} \frac{c_{a}}{\Delta_{ba}} (\mu^{t}_{a})^{2} + (\sigma^{t}_{a})^{2}).
\end{align}

Rewriting the overlap state evolution in a matrix form we get for $\bm{c} = (c_a)_{a \in [q]}$ that
\begin{equation}
	\operatorname{diag}(\sqrt{\bm{c}}) \bm{\mu}^{t + 1} =  \operatorname{diag}(\sqrt{\bm{c}}) \frac{1}{\bm{\Delta}} \operatorname{diag}(\sqrt{\bm{c}})  \left(\operatorname{diag}(\sqrt{\bm{c}})\bm{\mu}^{t}\right).
\end{equation}
If $\lambda(\bm{\Delta}) =   \left\|\operatorname{diag}(\sqrt{\bm{c}}) \frac{1}{\bm{\Delta}} \operatorname{diag}(\sqrt{\bm{c}})\right\|_{op} < 1,$ then this is a contraction with respect to the Euclidean norm, so there is a unique fixed point at $0$ when $\lambda(\bm{\Delta}) < 1$.  There is instability if $\lambda(\bm{\Delta}) > 1$, so we conjecture that this transition corresponds to the BBP transition for \eqref{eq:transformedY} as stated in Conjecture~\ref{conjecture:BBP}.
\begin{remark}\label{rem:scaling}
    In general, we can let $\gamma = \mathbb{E}_{x_0^\star} [(x_0^\star)^2]$ and consider the normalized matrix 
\[
\bm{\bar{Y}} = \frac{\bm{Y}}{\gamma} = \sqrt{\frac{1}{N}} \frac{\boldsymbol{x}^{\star}(\boldsymbol{x}^{\star})^{T}}{\gamma} + \boldsymbol{A} \odot \frac{\sqrt{\boldsymbol{\Delta}}}{\gamma} = \sqrt{\frac{1}{N}}\boldsymbol{\bar x}^{\star}(\boldsymbol{\bar x}^{\star})^{T} + \boldsymbol{A} \odot \sqrt{\boldsymbol{\bar \Delta}}
\]
for $\bm{\bar{x}} = \frac{\bm{x}}{\sqrt{\gamma}}$ and $\boldsymbol{\bar \Delta} =  \frac{\bm{\Delta}}{\gamma^2}$. Notice that the entries of $\bm{\bar{x}}$ now have unit variance. Under this setting, the transition of the transformation in \eqref{eq:transformedY} applied to $\bar{Y}$, which appears in \eqref{eq:transformedYgenvariance}, has transition at
\[
\lambda(\bm{\bar \Delta}) =   \left\|\operatorname{diag}(\sqrt{\bm{c}}) \frac{1}{\bm{\bar \Delta}} \operatorname{diag}(\sqrt{\bm{c}})\right\|_{op} = \mathbb{E}_{x_0^\star} [(x_0^\star)^2]^2 \left\|\operatorname{diag}(\sqrt{\bm{c}}) \frac{1}{\bm{ \Delta}} \operatorname{diag}(\sqrt{\bm{c}})\right\|_{op} < 1
\]
which is the generalized SNR defined in \eqref{eq:SNR}. 
\end{remark}


\section{Acknowledgments}
We thank Alice Guionnet \& Lenka Zdeborov\'a for valuable discussions. We acknowledge funding from the ERC Project LDRAM:  ERC-2019-ADG	Project 884584, and by the Swiss National Science Foundation grant SNFS OperaGOST, $200021\_200390$.

\newpage

\appendix

\section{Proof of Theorem~\ref{theorem: AMP inhomogeneous}}\label{app:theostate}

We now provide a rigorous proof of the result that was sketched in Section~\ref{sec:spiked}. This proof is essentially identical to \cite[Lemma~4.4]{DBLP:journals/corr/DeshpandeAM15}. Recall the iterates \eqref{eqn: AMP change spike} given by $\boldsymbol{s}^{0} = \boldsymbol{x}^{0} - \boldsymbol{\mu^{0}} \odot \boldsymbol{x}^{\star}$ and
\begin{equation}\label{eq:siterate}
\boldsymbol{s}^{t+1}= \left( \frac{1}{\sqrt{N}\sqrt{\boldsymbol{\Delta}}} \odot \boldsymbol{A} \right) f_{t} \left(\boldsymbol{s}^{t} + \boldsymbol{\mu^{t}} \odot \boldsymbol{x}^{\star}\right)-\bm{\mathrm{b}}_{t} \odot f_{t-1} \left(\boldsymbol{s}^{t-1} + \boldsymbol{\mu^{t - 1}} \odot \boldsymbol{x}^{\star}\right)     ,
\end{equation}
where $\boldsymbol{\mu}^{t} = (\mu_i^t)_{i \leq N} $ is given by the recursion
\[
\mu^{t + 1}_{i} = \mu^{t + 1}_{g(i)} =  \sum_{a \in [q]} \frac{c_{a}}{\tilde{\boldsymbol{\Delta}}_{ag(i)}}  \mathbb{E}_{x_{0}^{\star}, Z}[x_{0}^{\star} f_{t}^{a}\left(\mu^{t}_{a}x_{0}^{\star} + \sigma^{t}_{a}Z\right)]
\]
and $\boldsymbol{x}^{\star} = (x_i^\star)_{i \in [N]}$ is a vector with independent coordinates distributed according to $\mathbb{P}_0$.
By Lemma~\ref{lem:state_evoluton_centered}, for each $a \in [q]$, and any pseudo-Lipschitz function $\phi: \R \to \R \mapsto \R$ we have that almost surely
\begin{equation}\label{eq:pseudodouble}
\lim_{N \to \infty} \frac{1}{|C_a^N|} \sum_{i \in C_a^N} \phi( s_i^t, x_i^\star ) = \mathbb{E}_{x^\star_0,Z} \phi(  \sigma_{a}^{t}Z,x^\star_0 )     
\end{equation}
where
\[
(\sigma_b^{t + 1})^2 :=	\sum_{a=1}^{q} \frac{c_{a}}{\tilde{\boldsymbol{\Delta}}_{ab}}\mathbb{E}_Z \left[ (f^b_{t}(Z_{b}^{t}))^{2}\right] .
\]
as was defined in \eqref{eq:sigma_centered}. For any pseudo-Lipschitz function $\psi: \R \to \R$, we have $\phi(x,y) = \phi(x - \mu^t_a y)$ is also pseudo-Lipschitz, so \eqref{eq:pseudodouble} implies that
\begin{equation}\label{eq:limitingcentered}
\lim_{N \to \infty} \frac{1}{|C_a^N|} \sum_{i \in C_a^N} \psi( s_i^t + \mu_a^t  x_i^\star ) = \mathbb{E}_{x_0^\star,Z} \psi(  \sigma_{a}^{t}Z + \mu_a^t x_0^\star )    
\end{equation}
almost surely. 

Now let $\bx^t$ be the iterates from the spiked AMP iteration for the inhomogeneous Wigner matrix \eqref{eq:spikedWigner} we derived in \eqref{eq:xiterates1}
\begin{equation}\label{eq:xiterate}
\boldsymbol{x}^{t+1}= \bigg( \frac{1}{N \boldsymbol{\Delta}} \odot \boldsymbol{x^{\star}}\boldsymbol{(x^{\star})^{T}} \bigg) \boldsymbol{\hat{x}}^{t} + \bigg( \frac{1}{\sqrt{N} \sqrt{\boldsymbol{\Delta}}} \odot \bm{A} \bigg) \boldsymbol{\hat{x}}^{t}-\boldsymbol{\mathrm{b}}_{t} \odot \boldsymbol{\hat{x}}^{t-1}.  
\end{equation}
It now suffices to show that for fixed $t$ and all $a \in [q]$ that
\begin{equation}\label{eq:goal}
\lim_{N \to \infty} \frac{1}{|C_a^N|} \sum_{i \in C_a^N} ( \psi( s_i^t + \mu_a^t x_i^\star ) - \psi( x_i^t) ) = 0    
\end{equation}
almost surely. This will imply that $ s_i^t + \mu_a^t x_i^\star$ and $ x_i^t$ have the same asymptotic distribution which finish the proof of Theorem~\ref{theorem: AMP inhomogeneous} by \eqref{eq:limitingcentered}.

We now prove \eqref{eq:goal}. Since $\psi$ is $L$-pseudo-Lipschitz we have
\begin{align*}
| \psi( s_i^t + \mu_a^t x_i^\star ) - \psi( x_i^t) | &\leq L ( 1 + |s_i^t + \mu_a^t x_i^\star| + | x_i^t|  ) | s_i^t + \mu_a^t x_i^\star - x_i^t|    
\\&\leq 2 L | s_i^t + \mu_a^t x_i^\star - x_i^t|   ( 1 + |s_i^t + \mu_a^t x_i^\star|  + | s_i^t + \mu_a^t x_i^\star - x_i^t|    ) .
\end{align*}
The Cauchy--Schwarz inequality implies that
\begin{align*}
    &\bigg| \frac{1}{|C_a^N|} \sum_{i \in C_a^N} ( \psi( s_i^t + \mu_a^t x_i^\star ) - \psi( x_i^t) ) \bigg| 
    \\&\leq \frac{2 L}{C_a^N}   (\sqrt{C_a^N} \| \bm{s}_a^t + \mu_a^t \bm{x}_a^\star - \bm{x}_a^t\|_2 + \|\bm{s}_a^t + \mu_a^t \bm{x}_a^\star\|_2 \| \bm{s}_a^t + \mu_a^t \bm{x}_a^\star - \bm{x}_a^t\|_2  + \| \bm{s}_a^t + \mu_a^t \bm{x}_a^\star - \bm{x}_a^t\|_2^2    )  
\end{align*}
where $\bm{s}_a^t = (s_i^t)_{i \in C_a^N} \in \R^{|C_a^N|}$ , $\bm{x}_a^t = ( x_i^t )_{i \in C_a^N} \in \R^{|C_a^N|}$. Therefore, to prove \eqref{eq:goal} it suffices to prove that for all $t \geq 0$,
\begin{eqnarray}
\lim_{N \to \infty}  \frac{1}{|C_a^N|} \| \bm{s}_a^t + \mu_a^t \bm{x}_a^\star - \bm{x}_a^t\|^2_2  \to 0 \label{eq:firstest}\\
\limsup_{N \to \infty} \frac{1}{|C_a^N|} \|\bm{s}_a^t + \mu_a^t \bm{x}_a^\star\|^2_2 \to 0 \label{eq:secondest}
\end{eqnarray}

Clearly, if we initialize $\bm{x}^0$, $\bm{s}^0$ at zero then \eqref{eq:firstest} and \eqref{eq:secondest} are satisfied by our state evolution equations \eqref{eq:state}.
Notice that \eqref{eq:secondest} follows directly from \eqref{eq:limitingcentered} applied to the square function. We use here that we assumed that the second moment of $x^\star$ is finite.

We now focus on proving \eqref{eq:firstest} through strong induction. By definition of the iterates \eqref{eq:siterate} and \eqref{eq:xiterate}, 
\begin{align*}
    &(\bm{s}_a^t + \mu_a^t \bm{x}_a^\star - \bm{x}_a^t)
    \\&= \bigg[ \left( \frac{1}{\sqrt{N}\sqrt{\boldsymbol{\Delta}}} \odot \boldsymbol{A} \right) f_{t-1} \left(\boldsymbol{s}^{t-1}     + \boldsymbol{\mu^{t-1}} \odot \boldsymbol{x}^{\star}\right) - \bigg( \frac{1}{\sqrt{N} \sqrt{\boldsymbol{\Delta}}} \odot \bm{A} \bigg)f_{t-1} (\bm{x}^{t-1})
    \\&\quad + \mu^t \odot \bx^\star - \bigg( \frac{1}{N \boldsymbol{\Delta}} \odot \boldsymbol{x^{\star}}\boldsymbol{(x^{\star})^{T}} \bigg) f_{t-1} (\bm{x}^{t-1})
    \\&\quad + \boldsymbol{\mathrm{b}}^x_{t-1} \odot f_{t - 2}(\boldsymbol{x}^{t-2}) - \bm{\mathrm{b}}^s_{t-1} \odot f_{t-2} \left(\boldsymbol{s}^{t-2} + \boldsymbol{\mu^{t - 2}} \odot \boldsymbol{x}^{\star}\right) \bigg]_{i \in C_a^N}
\end{align*}
where $[\cdot]_i$ corresponds to the $i$th row of a vector and $\boldsymbol{\mathrm{b}}^x_{t-1}$ and $\bm{\mathrm{b}}^s_{t}$ are the Onsager terms defined in \eqref{eqn: Onsager vector} with respect to $\bm{x}^{t-1}$ and $\bm{s}^{t-1}$ respectively. 
The Cauchy--Schwarz inequality and Jensen's inequality imply that there exists some universal constant $C$ such that
\begin{align*}
   & \frac{1}{|C_a^N|}\| \bm{s}_a^t + \mu_a^t \bm{x}_a^\star - \bm{x}_a^t\|^2_2  
    \\&\leq \frac{C}{|C_a^N|}  \sum_{i \in C_a^N} \frac{1}{N} \left\|  \bigg[\frac{1}{\sqrt{\boldsymbol{\Delta}}} \odot \boldsymbol{A} \bigg]_i \right\|_2^2 \| [ f_{t-1} \left(\boldsymbol{s}^{t-1}     + \boldsymbol{\mu^{t-1}} \odot \boldsymbol{x}^{\star}\right) - f_{t-1} (\bm{x}^{t-1}) ]_i \|_2^2
    \\&\quad + \frac{C}{|C_a^N|}  \sum_{i \in C_a^N} \bigg(  \mu_a^t  - \bigg[ \frac{1}{N \boldsymbol{\Delta}}  (f_{t-1} (\bm{x}^{t-1}) \odot \bm{x}^\star) \bigg]_i \bigg)^2 (x^\star_i)^2
    \\&\quad + \frac{C}{|C_a^N|}  \sum_{i \in C_a^N}  ([\mathrm{b}^x_{t-1}]_i - [\mathrm{b}^s_{t-1}]_i)^2 [f_{t - 2}(\boldsymbol{s}^{t-2} + \boldsymbol{\mu^{t - 2}} \odot \boldsymbol{x}^{\star})]^2_i 
    \\&\quad + \frac{C}{|C_a^N|}  \sum_{i \in C_a^N}  [\bm{\mathrm{b}}^x_{t-1}]^2_i ( [f_{t - 2}(\boldsymbol{x}^{t-2})]_i - [f_{t - 2}(\boldsymbol{s}^{t-2} + \boldsymbol{\mu^{t - 2}} \odot \boldsymbol{x}^{\star})]_i)^2
\end{align*}
We now control each term separately. 
\begin{enumerate}
    \item To control the first term, notice that the matrix $\frac{1}{N} \bigg[\frac{1}{\sqrt{\boldsymbol{\Delta}}} \odot \boldsymbol{A} \bigg]$ has iid entries within blocks and the sizes of the blocks diverge with the dimension, so we can control the sums of the squares of within each block using standard operator norm bounds \cite{AGZ}. The first term vanishes in the limit because $f$ is pseudo-Lipschitz so we can apply the induction hypothesis bound which controls \eqref{eq:firstest} at time $t - 1$. 

\item To control the second term, notice that for $i \in C_a^N$ by Lemma~\ref{lem:state_evoluton_centered} applied to the pseudo-Lipschitz function $y f_{t-1}(x)$ that 
\[
\bigg[ \frac{1}{N \boldsymbol{\Delta}}  (f_{t-1} (\bm{x}^{t-1}) \odot \bm{x}^\star) \bigg]_i \to \mu_a
\]
almost surely. This implies that the average of such terms vanishes since we assumed that the second moment $\mathbb{E} [x_0^\star]^2$ is finite.

\item To control the third and fourth terms, we can expand the definition of the Onsager terms and use the assumption that $f'$ is pseudo-Lipschitz and almost surely bounded. Both terms vanish because our strong induction hypothesis gives us control of \eqref{eq:firstest} at time $t - 2$. 
\end{enumerate}
Since all terms vanish in the limit, we have proven \eqref{eq:firstest} for all $a \in [q]$, which finishes the proof of statement \eqref{eq:goal} and the proof of Theorem~\ref{theorem: AMP inhomogeneous}.

\begin{figure}[t]
	\begin{center}
        
        		\includegraphics[width=6cm]{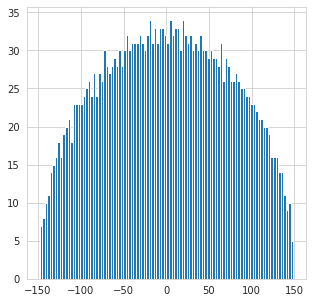}
          \includegraphics[width=6cm]{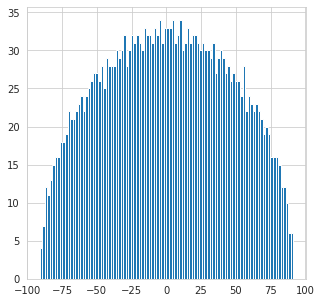}

		\caption{Illustration of the spectrum of $ Y \in \R^{2500 \times 2500}$ evaluated at noise profiles with {\rm snr} $\lambda(\bm{\Delta}) = 0.7$ (left, before the transition) and on the left and $1.8$ on the right (after the transition). There is no outlying eigenvalue in contrast to the transformed matrix: the transition for a naive spectral method is sub-optimal.}
  \label{fig:spectrumY}
	\end{center}
\end{figure}
\begin{figure}[h]
	\begin{center}
        
        		\includegraphics[width=6cm]{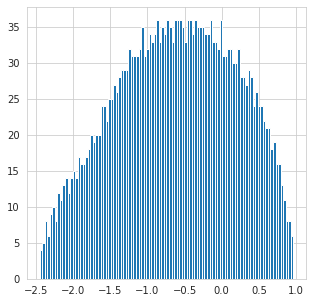}
          \includegraphics[width=6cm]{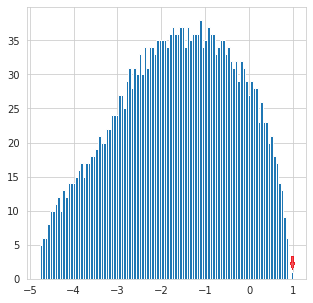}

		\caption{Illustration of the spectrum of $\tilde Y \in \R^{2500 \times 2500}$ evaluated at noise profiles with {\rm snr} $\lambda(\bm{\Delta}) = 0.7$ (left, before the transition) and on the left and $1.8$ on the right (after the transition), with the outlying eigenvector correlated with the spike arises at eigenvalue one. This is at variance with the results of the naive method in Fig.\ref{fig:spectrumY}}
  \label{fig:spectrumtilde}
	\end{center}
\end{figure}

\section{Comparison with a naive PCA spectral method}
In this appendix, we wish to show how the spectral method we propose differs, in practice, from a naive PCA. We provide an example of the spectrums of $\bm{Y}$ and $\bm{\tilde {Y}}$ before and after the transition at $\rm{SNR}(\bm{\Delta}) = 1$. In  Figure~\ref{fig:spectrumY} there is no clear separation of the extremal eigenvalue of $\bm{Y}$ from the bulk around this transition. This is in contrast to  Figure~\ref{fig:spectrumtilde} where there is an extremal eigenvalue of $\bm{\tilde Y}$ appearing at value one.

\end{document}